\documentclass[a4paper]{article}

\usepackage{local}

\usepackage{url}

\usepackage[english, linesnumbered]{algorithm2e}
\SetKwBlock{Deb}{begin}{end}
\SetKwFor{ForEach}{for each}{do}{end}

\usepackage{graphicx}

\begin{document}

\title{
    Belief revision in the propositional closure \\ of a qualitative algebra\footnote{
        This technical report constitutes an extended version of \citet{dufour2014_REVISOR_PCQA_short_version}.
    }}
\author[1,2,3]{Valmi Dufour-Lussier}
\author[1,2,3]{Alice Hermann}
\author[4]{\\Florence Le Ber}
\author[1,2,3]{Jean Lieber}
\affil[1]{
        Universit\'e de Lorraine, LORIA, UMR 7503
        \par
        54506 Vand\oe{}uvre-l\`es-Nancy, France
        \par
        \texttt{\textit{first-name}.\textit{surname}@loria.fr}
}
\affil[2]{ CNRS --- 54506 Vand\oe{}uvre-l\`es-Nancy, France}
\affil[3]{ Inria --- 54602 Villers-l\`es-Nancy, France}
\affil[4]{
        ICube -- Universit\'e de Strasbourg/ENGEES, CNRS
        \par
        67412 Illkirch, France
        \par
        \texttt{florence.leber@engees.unistra.fr}
}
\date{}
\maketitle

\begin{abstract}
  Belief revision is an operation that aims at modifying old beliefs so that they
   become consistent with new ones.
  The issue of belief revision has been studied in various formalisms, in particular,
   in qualitative algebras (QAs) in which the result is a disjunction of belief bases
   that is not necessarily representable in a QA.
  This motivates the study of belief revision in formalisms extending QAs,
   namely, their propositional closures:
   in such a closure, the result of belief revision belongs
   to the formalism.
  Moreover, this makes it possible to define a contraction operator thanks to the Harper
   identity.
  Belief revision in the propositional closure of QAs is studied, an algorithm
   for a family of revision operators is designed, and
   an open-source implementation is made freely available on the web.

\end{abstract}

\section{Introduction}

Belief revision is an operation of belief change that consists
 in modifying minimally old beliefs so that they become consistent
 with new beliefs~\citep{agm-85}.
One way to study this issue following a knowledge representation angle
 is to consider a formalism and to study some belief revision operators
 defined on it:
 how they are defined and how they can be implemented.

In particular, it is rather simple to define a revision operator on a
 qualitative algebra (such as the Allen algebra) by
 reusing the work of~\citet{condotta10} about the related issue of
 belief merging.
The result of such a belief revision is a set of belief bases to
 be interpreted disjunctively, and which is not necessarily representable
 as a sole belief base:
 qualitative algebras are not closed under disjunction.

This gives a first motivation for the study of belief revision in the
 propositional closure of a qualitative algebra:
 the revision operator in such a closure gives a result necessarily
 representable in the formalism.

The first section of the paper contains some preliminaries about various
 notions used throughout the paper;
 this section is rather long since it contains notions from which a big part
 of the rest of the paper ensues,
 e.g. propositional closure of a formalism, qualitative algebras, and
 belief revision based on distances.
%
Then, the paper describes some motivations about the study of belief revision
 in the propositional closure of a qualitative algebra.
The next section briefly describes some properties of such a formalism.
Finally, an algorithm and an implementation of this algorithm for a revision
 operator in the propositional closure of a qualitative algebra are
 presented with some detailed examples.

\section{Preliminaries}

\subsection{Main terminology and assumptions about knowledge representation formalisms}

A (knowledge representation) \emph{formalism} is a pair $(\logique, {\models})$
 where $\logique$ is a language and $\models$ is a binary relation on $\logique$.
A \emph{formula} $\varphi$ is an element of $\logique$.
$\models$ is called the entailment relation.
For $\varphi_1, \varphi_2\in\logique$,
 $\varphi_1\equiv\varphi_2$ means that
 $\varphi_1\models\varphi_2$ and $\varphi_2\models\varphi_1$
 and is read ``$\varphi_1$ and $\varphi_2$ are equivalent''.

The entailment relation of the formalisms used in this paper can always
 be characterized as follows---%
 according to a model-theoretic semantics with a class of interpretations
  that is a set:
It is assumed that there is a set $\INTERPRETATIONS$ whose elements are called the
 \emph{interpretations}.
There is a relation also denoted by $\models$ on $\INTERPRETATIONS\times\logique$.
If $\interpretationOmega\models\varphi$,
 for $\interpretationOmega\in\INTERPRETATIONS$ and $\varphi\in\logique$,
 $\interpretationOmega$ is said to be a \emph{model} of $\varphi$.
The set of models of $\varphi$ is called by $\Mod(\varphi)$.
Therefore, the entailment relation is defined as follows:
 for $\varphi_1, \varphi_2\in\logique$,
 $\varphi_1\models\varphi_2$ if $\Mod(\varphi_1)\subseteq\Mod(\varphi_2)$.
From that, it can be implied that
 $\varphi_1\equiv\varphi_2$ is equivalent to $\Mod(\varphi_1)=\Mod(\varphi_2)$.

A formula $\varphi$ is \emph{consistent} (or satisfiable)
 if $\Mod(\varphi)\neq\emptyset$.
$\varphi$ is a \emph{tautology} if $\Mod(\varphi)=\INTERPRETATIONS$.

$\logique$ is assumed to be \emph{closed under conjunction},
 which means that
 for any $\varphi_1,\varphi_2\in\logique$ there exists $\varphi\in\logique$
 such that $\Mod(\varphi)=\Mod(\varphi_1)\cap\Mod(\varphi_2)$;
 $\varphi$ is unique up to equivalence and is written
 $\varphi_1\land\varphi_2$.
$\land$ is associative wrt equivalence, so one can write
 $\varphi_1\land\varphi_2\land\varphi_3$:
 no matter where the parentheses are placed, the formula will have
 the same set of models.
Thus, the formalism is simplified, without loss of expressiveness,
 by removing such useless parentheses.
It is also commutative wrt equivalence.

A \emph{knowledge base} $\baseConnaissances$ is a finite subset of $\logique$.
It is assimilated as the conjunction of its elements.

A formalism $(\logique, {\models})$ is \emph{closed under disjunction} if
 for any $\varphi_1,\varphi_2\in\logique$ there exists $\varphi\in\logique$
 such that $\Mod(\varphi)=\Mod(\varphi_1)\cup\Mod(\varphi_2)$;
 then $\varphi$ is unique up to equivalence and is written
 $\varphi_1\lor\varphi_2$.
$\lor$ is commutative and associative wrt equivalence.

A formalism $(\logique, {\models})$ is \emph{closed under negation} if
 for any $\varphi_1\in\logique$ there exists $\varphi\in\logique$
 such that $\Mod(\varphi)=\INTERPRETATIONS\setminus\Mod(\varphi_1)$;
 then $\varphi$ is unique up to equivalence and is denoted by
 $\lnot\varphi_1$.

A formalism $(\logique, {\models})$ is \emph{propositionally closed}
 if it is closed under conjunction and negation.
In this situation, it is also closed under disjunction
 (consider $\varphi_1\lor\varphi_2$ as an abbreviation for
  $\lnot(\lnot\varphi_1\land\lnot\varphi_2)$).

The \emph{propositional closure} of a formalism $(\logique, {\models})$
 is the formalism $(\logiqueCP, {\modelsCP})$
 such that $\logiqueCP$ is the smallest superset of $\logique$
 verifying:
 \begin{itemize}
 \item
   If $\varphi_1, \varphi_2\in\logiqueCP$ then $\varphi_1\land\varphi_2\in\logiqueCP$;
 \item
   If $\varphi_1, \varphi_2\in\logiqueCP$ then $\varphi_1\lor\varphi_2\in\logiqueCP$;
 \item
   If $\varphi\in\logiqueCP$ then $\lnot\varphi\in\logiqueCP$;
 \end{itemize}
  and $\modelsCP$ is the entailment relation defined by the
  $\ModCP$ function which extends $\Mod$ on $\logiqueCP$ and is such
  that
  $\ModCP(\varphi_1\land\varphi_2)=\ModCP(\varphi_1)\cap\ModCP(\varphi_2)$,
  $\ModCP(\varphi_1\lor\varphi_2)=\ModCP(\varphi_1)\cup\ModCP(\varphi_2)$,
  and
  $\ModCP(\lnot\varphi_1)=\INTERPRETATIONS\setminus\ModCP(\varphi_1)$
  (for any $\varphi_1, \varphi_2\in\logiqueCP$).
The meta-language expression $\varphi_1\equivCP\varphi_2$ means that
 $\varphi_1\modelsCP\varphi_2$ and $\varphi_2\modelsCP\varphi_1$.
In the following, when the context is explicit, hats will be
 omitted
 ($\models$ and $\equiv$ instead of $\modelsCP$ and $\equivCP$).

Let us consider a propositionally closed formalism $(\logique, {\models})$.
An \emph{atom} is a formula without any occurrence of the symbols
 $\lnot$, $\lor$ and $\land$
 (e.g. in propositional logic, atoms are propositional variables).
A \emph{literal} either is an atom (positive literal)
 or is of the form $\lnot{}a$ where $a$ is an atom (negative literal).
A formula is under disjunctive normal form (\emph{DNF})
 if it is a disjunction of conjunctions of literals.
Every formula $\varphi$ is equivalent to a formula under DNF.
To prove this, first, it can be proven that the following equivalences
 hold:
 \begin{equation}
   \begin{split}
     \varphi\land(\varphi_1\lor\varphi_2\lor\ldots\varphi_n)
     &\equiv
     (\varphi\land\varphi_1)\lor\ldots\lor(\varphi\land\varphi_n)
     \\
     \lnot(\varphi_1\land\ldots\land\varphi_n)
     &\equiv
     \lnot\varphi_1\lor\ldots\lor\lnot\varphi_n
     \\
     \lnot(\varphi_1\lor\ldots\lor\varphi_n)
     &\equiv
     \lnot\varphi_1\land\ldots\land\lnot\varphi_n
     \\
     \lnot\lnot\varphi
     &\equiv
     \varphi
   \end{split}
   \label{eq:equivalences-pour-DNF}
 \end{equation}
 for any $\varphi, \varphi_1, \ldots, \varphi_n\in\logique$.
Then, applying these equivalences from left to right until
 it is not possible to do this, starting with $\varphi$,
 results in a formula under DNF equivalent to $\varphi$.

\subsection{Distance functions}
\def\unensemble{X}

A distance function on a set $\unensemble$
 is a function $\dist : \unensemble^2\rightarrow\Reels_+$
 (where $\Reels_+$ is the set of non negative real numbers) verifying
 the separation axiom
 ($\dist(x, y)=0$ iff $x=y$),
 the symmetry
 ($\dist(x, y)=\dist(y, x)$)
 and the triangular inequality
 ($\dist(x, z)\leq\dist(x, y)+\dist(y, z)$).

Given $A, B\in\EnsembleParties{\unensemble}$ and $y\in\unensemble$,
 $\dist(A, y)$ is an abbreviation for
 $\inf_{x\in{}A}\dist(x, y)$ and
 $\dist(A, B)$ is an abbreviation for
 $\inf_{x\in{}A, y\in{}B}\dist(x, y)$.

\subsection{Qualitative algebras}

Qualitative algebras (QAs) are formalisms that are widely used
 for representation depending on time and/or on space~\citep{stock97}.
Formulas built upon QAs are closed under conjunction,
 though the symbol $\land$ is not systematically used.
Some of the usual notations and conventions of QAs
 are changed to better fit the scope of this paper.
In particular, the representation of knowledge by graphs
 (namely, qualitative constraint networks)
 is not well-suited here, because of the propositional closure
 introduced afterwards.

First, the Allen algebra is introduced: 
 it is one of the most famous QAs and it will be used in
 our examples throughout the paper.
Then, a general definition of QAs is given.

\subsubsection{The Allen algebra}
 is used for representing relations between time intervals~\citep{allen83cacm}.
A formula of the Allen Algebra can be seen as a conjunction of constraints, where
 a constraint is an expression of the form
 $\varx\relAQ\vary$ stating that the interval $\varx$ is related to
 the interval $\vary$ by the relation $\relAQ$.
$13$ base relations are introduced (cf. figure~\ref{fig:Allen}(a));
 a relation $\relAQ$ is either one of these base relations
 or the union of base relations $\relAQNum1$, \ldots, $\relAQNum{m}$
 denoted by ${\relAQNum1}\ourel\ldots\ourel{\relAQNum{m}}$.

For example, if one wants to express that the maths course
 is immediately before the physics course
 which is before the English course
 (either with a time lapse, or immediately before it), one can write the
 formula:
 \begin{align*}
   \maths \mAQ \physique
   \quad\land\quad
   \physique \mathrel{{\bAQ}\ourel{\mAQ}} \anglais
 \end{align*}

$\logiqueAllen$ is the set of the formulas of the Allen algebra.

\begin{figure}
    \begin{center}
        {
  \def\intx{$\rule[1.5mm]{20mm}{1.5mm}$}
  \def\inty#1#2{{\textcolor{gray!50}{\hspace{#1mm}$\rule[0mm]{#2mm}{1.5mm}$}}}
  \def\ligne#1#2#3#4{\makebox[0mm][l]{\intx}\inty{#1}{#2} & $\nomrelAQ{#3}$ & \emph{#4}}
  \def\rien#1{#1} 
  %
    \begin{tabular}{l c l}
      \ligne{25}{15}{b}{is \rien{b}efore}
      \\[1.5mm]
      \ligne{20}{15}{m}{\rien{m}eets}
      \\[1.5mm]
      \ligne{15}{15}{o}{\rien{o}verlaps}
      \\[1.5mm]
      \ligne{0}{25}{s}{\rien{s}tarts}
      \\[1.5mm]
      \ligne{-5}{30}{d}{is \rien{d}uring}
      \\[1.5mm]
      \ligne{-5}{25}{f}{\rien{f}inishes}
      \\[1.5mm]
      \ligne{0}{20}{eq}{\rien{eq}uals}
    \end{tabular}
  %
}

    \end{center}
    \centerline{(a) Intuitive meaning.}
    \par\vspace{2mm}
    $\Domaine$ is the set of closed and bounded intervals $[a, b]$
     of $\Rationnels$ (the set of rational numbers) such that
     $a<b$.
    The base relations are defined as follows, with
     $d_1, d_2\in\Domaine$,
     $d_1=[a_1, b_1]$, $d_2=[a_2, b_2]$: 
    \begin{align*}
      (d_1, d_2)\in\interprel{\eqAQ} \qquad &\text{if $a_1=a_2$ and $b_1=b_2$}
      \\
      (d_1, d_2)\in\interprel{\bAQ} \qquad &\text{if $b_1<a_2$}
      \\
      (d_1, d_2)\in\interprel{\mAQ} \qquad &\text{if $a_2=b_1$}
      \\
      (d_1, d_2)\in\interprel{\oAQ} \qquad &\text{if $a_1<a_2$, $a_2<b_1$ and $b_1<b_2$}
      \\
      (d_1, d_2)\in\interprel{\sAQ} \qquad &\text{if $a_1=a_2$ and $b_1<b_2$}
      \\
      (d_1, d_2)\in\interprel{\fAQ} \qquad &\text{if $a_1>a_2$ and $b_1=b_2$}
      \\
      (d_1, d_2)\in\interprel{\dAQ} \qquad &\text{if $a_1>a_2$ and $b_1<b_2$}
      \\
      {\biAQ}={\inverserel{\bAQ}}
      \qquad
      &{\miAQ}={\inverserel{\mAQ}}
      \qquad
      {\oiAQ}={\inverserel{\oAQ}}
      \\
      {\siAQ}={\inverserel{\sAQ}}
      \qquad
      &{\fiAQ}={\inverserel{\fAQ}}
      \qquad\quad
      {\diAQ}={\inverserel{\dAQ}}
    \end{align*}
    %
    \centerline{(b) Semantics based on a domain.}
  \caption{The base relations of $\logiqueAllen$.\label{fig:Allen}}
\end{figure}

\subsubsection{Qualitative algebras}
 in general are defined below, first by their syntax and then by their semantics.
Finally, some inference mechanisms are described.

\paragraph{Syntax.}
A finite set of symbols $\RelationsAQBases$ is given
 (with $|\RelationsAQBases|\geq2$).
A \emph{base relation} is an element of $\RelationsAQBases$.
A \emph{relation} is an expression of the
 form ${\relAQNum1}\ourel\ldots\ourel{\relAQNum{m}}$
 ($m\geq0$),
 such that a base relation occurs at most once in a relation
 and the order is irrelevant
 (e.g. ${\relAQNum1}\ourel{\relAQNum2}$ and ${\relAQNum2}\ourel{\relAQNum1}$
  are equivalent expressions).
The set of relations is denoted by $\RelationsAQ$,
 which is of cardinality $|\RelationsAQ|=2^{|\RelationsAQBases|}$.
The relation in which all the base relations occur is named
 $\toutAQ$.
The relation ${\relAQNum1}\ourel\ldots\ourel{\relAQNum{m}}$ with $m=0$
 is named $\aucunAQ$.

A finite set of symbols $\Variables$, disjoint from $\RelationsAQBases$, is given.
A \emph{(qualitative) variable} is an element of $\Variables$.

A \emph{constraint} is an expression of the form $\varx\relAQ\vary$
 where $\varx, \vary\in\Variables$ and ${\relAQ}\in\RelationsAQ$.

A \emph{formula} $\varphi$ is a conjunction of $n$ constraints ($n\geq1$):
 $\varxNum1\relAQNum1\varyNum1\;\land\;\ldots\;\land\;\varxNum{n}\relAQNum{n}\varyNum{n}$.
A constraint of $\varphi$ is one of the constraints of this conjunction.
Let $\logiqueQA$ be the set of the formulas of the considered QA.
The atoms of $\logiqueQA$ are the constraints.

A formula $\varphi\in\logiqueQA$ is under normal form if
 for every $\varx, \vary\in\Variables$ with $\varx\neq\vary$,
 there is exactly one ${\relAQ}\in\RelationsAQ$ such that
 $\varx\relAQ\vary$ is a constraint of $\varphi$.
Then, this relation $\relAQ$ is denoted by $\relDe{\varphi}(\varx, \vary)$.

A \emph{scenario} $\scenario$ is a formula under normal form such that,
 for every variables $\varx$ and $\vary$, $\varx\neq\vary$,
 $\relDe{\scenario}(\varx, \vary)\in\RelationsAQBases$.
Therefore, there are $|\RelationsAQBases|^{|\Variables|\times(|\Variables|-1)}$
 scenarios.
Given a formula $\varphi$ under normal form, $\Scenarios(\varphi)$
 is the set of scenarios obtained by substituting each constraint
 $\varx\mathrel{{\relAQNum1}\ourel\ldots\ourel{\relAQNum{m}}}\vary$ ($m\geq2$)
 of $\varphi$
 with a constraint $\varx\relAQNum{k}\vary$ ($1\leq{}k\leq{}m$).

\paragraph{Semantics.}

The semantics will be described twice.
The two descriptions correspond to the same entailment
 relation, but serve different purposes.
The first one gives a semantics based on a domain $\Domaine$
 on which the relations are interpreted, but the class of
 interpretations for this semantics is difficult to use for
 the purpose of the paper.
This motivates a second semantics, defining a finite set $\INTERPRETATIONS$
 of interpretations, where an interpretation is a consistent scenario
 and on which a distance function can be easily defined.

\subparagraph{Semantics based on a domain $\Domaine$.}
The semantics of the Allen algebra given in figure~\ref{fig:Allen}(b)
 exemplifies this section.

Let $\Domaine$ be a nonempty set,
 and let $\interprel{\cdot}$ be a mapping that associates to
 each ${\relAQ}\in\RelationsAQBases$
 a relation $\interprel{\relAQ}$ on $\Domaine$
 ($\interprel{\relAQ}\subseteq\Domaine^2$) such that:
 \begin{itemize}
 \item
   $\interprel{\RelationsAQBases}=\{{\interprel{\relAQ}} ~|~ {\relAQ}\in\RelationsAQBases\}$
   is a partition of $\Domaine^2$:
   for each $(d, e)\in\Domaine^2$ there is exactly one
   ${\relAQ}\in\RelationsAQBases$ such that $(d, e)\in{\interprel{\relAQ}}$.
   Furthermore, each $\interprel{\relAQ}\in\interprel{\RelationsAQBases}$
    is nonempty.
 \item
   For each ${\relAQ}\in\RelationsAQBases$ there exists exactly one
    ${\relAQs}\in\RelationsAQBases$ such that ${\interprel{\relAQs}}$ is
    the inverse of the relation $\interprel{\relAQ}$.
   In the following, $\relAQs$ is denoted by $\inverserel{\relAQ}$.
 \item
   There is a base relation, denoted by $\eqAQ$, that is interpreted
    as the equality on $\Domaine$:
    $\interprel{\eqAQ}=\{(d, d) ~|~ d\in\Domaine\}$.
   $\eqAQ$ is its own inverse: $\inverserel{\eqAQ}={\eqAQ}$.
 \end{itemize}
This mapping is extended on $\RelationsAQ$ as follows:
 \begin{align*}
   \text{if ${\relAQ}\in\RelationsAQ$ and ${\relAQ}={\relAQNum1}\ourel\ldots\ourel{\relAQNum{m}}$}
   \text{ then } \interprel{\relAQ}=\interprel{\relAQNum1}\cup\ldots\cup\interprel{\relAQNum{m}}
 \end{align*}
In other words: $(d, e)\in\interprel{\relAQ}$ iff
 there exists $i\in\{1, \ldots, m\}$
 such that $(d, e)\in\interprel{\relAQNum{i}}$.

An interpretation $\interpretation$ is a mapping from $\Variables$
 to $\Domaine$.
$\interpretation$ is a model of $\varx\relAQ\vary$
 if \linebreak $(\interpretation(\varx), \interpretation(\vary))\in{\interprel{\relAQ}}$.
$\interpretation$ satisfies a conjunction of constraints if it
 satisfies every constraint in the conjunction.
A formula $\varphi$ is consistent if there exists an interpretation
 satisfying it.
Finally, $\varphi_1\models\varphi_2$ if every interpretation that satisfies
 $\varphi_1$ also satisfies $\varphi_2$.

According to this semantics,
 any constraint of the form $\varx\toutAQ\vary$ is a tautology and
 any constraint of the form $\varx\aucunAQ\vary$ is inconsistent.
Moreover, any formula $\varphi$ is equivalent to a formula $\varphi'$
 under normal form.\footnote{%
   This can be proven by considering,
    for any $\varx, \vary\in\Variables$, $\varx\neq\vary$,
    the set $R_{\varx\vary}$ of relations $\relAQ$ such that $\varx\relAQ\vary$
    is a constraint of $\varphi$.
   If $R=\emptyset$, let $C_{\varx\vary}$ be the constraint
    $\varx\toutAQ\vary$.
   Else, let ${\relAQ}_{\varx\vary}$ be the relation constituted of the base
    relations that occur in all relations of $R_{\varx\vary}$
    (for example, if
     $R_{\varx\vary}=\{{{\bAQ}\ourel{\mAQ}\ourel{\oAQ}},~~ {{\mAQ}\ourel{\oAQ}\ourel{\sAQ}}\}$
     then ${\relAQ}_{\varx\vary}={\mAQ}\ourel{\oAQ}$).
    Then, $C_{\varx\vary}$ is the constraint $\varx\mathrel{{\relAQ}_{\varx\vary}}\vary$.
    Finally, the formula
     $\displaystyle\bigwedge_{\varx, \vary\in\Variables, \varx\neq\vary}C_{\varx\vary}$
     is a formula under normal form equivalent to $\varphi$.}
Thus, in the following of the paper, all the formulas of $\varphi$
 are assumed to be under normal form, without lost of expressiveness.

\subparagraph{Semantics defined by consistent scenarios.}
%
The semantics can be characterized a posteriori thanks to
 consistent scenarios.

Let $\INTERPRETATIONS$ be the set of consistent scenarios on the
 variables of $\Variables$.
It can be easily proven that
 $|\INTERPRETATIONS|\leq|\RelationsAQBases|^{|\Variables|\times(|\Variables|-1)/2}$:
 if $\varx\relAQ\vary$ is a constraint of a consistent scenario $\scenario$
 then $\vary\inverserel{\relAQ}\varx$ is also a constraint of $\scenario$.

Let $\Mod : \logique\rightarrow\EnsembleParties{\INTERPRETATIONS}$ be defined
 by
 \begin{equation*}
   \Mod(\varphi)=\{\scenario\in\INTERPRETATIONS~|~\scenario\models\varphi\}
 \end{equation*}
 for $\varphi\in\logique$, where $\models$ is the entailment relation
 defined below, thanks to the semantics based on a domain.

$\Omega$ and $\Mod$ make it possible to define a semantics on $\logique$
 which coincides with the semantics based on a domain
 (hence the same entailment relation $\models$).
However, this second semantics is more practical to use for the
 definition of revision on QAs.

\paragraph{Inferences.}
The main inference about QAs used in this paper is
 the test of consistency.

It is usually implemented in the following way.
Properties on formulas named \emph{arc consistency} and \emph{path consistency}
 are defined.
Having those properties are a necessary condition and,
 in most algebras,
 a sufficient condition for
 scenarios to be consistent
 (a scenario $\scenario$ is consistent iff it is arc-consistent and path-consistent).
A way to test if $\varphi\in\logiqueQA$ is consistent
 is to test whether there exists $\scenario\in\Scenarios(\varphi)$
 that is consistent.

A formula $\varphi\in\logiqueQA$ is arc-consistent if:
 \begin{itemize}
  \item
   For all $2$ variables $\varx, \vary\in\Variables$,
    $\relDe{\varphi}(\varx, \vary)\neq{\aucunAQ}$.
  \item
    For all $2$ variables $\varx, \vary\in\Variables$,
     $\relDe{\varphi}(\varx, \vary)=\inverserel{\relDe{\varphi}(\vary, \varx)}$.
 \end{itemize}

The definition of path consistency is based on a binary operation
 on $\RelationsAQ$, written $\fcomp$.
It is defined on $\RelationsAQBases$
 (for example by a $|\RelationsAQBases|\times|\RelationsAQBases|$ table)
 and extended on
 $\RelationsAQ$ thanks to the equalities
 \begin{align*}
   ({\relAQNum1}\ourel\ldots\ourel{\relAQNum{m}})\fcomp{\relAQs}
   &= ({\relAQNum1}\fcomp{\relAQs})\ourel\ldots\ourel({\relAQNum{m}}\fcomp{\relAQs})
   \\
   {\relAQs}\fcomp({\relAQNum1}\ourel\ldots\ourel{\relAQNum{m}})
   &= ({\relAQs}\fcomp{\relAQNum1})\ourel\ldots\ourel({\relAQs}\fcomp{\relAQNum{m}})
 \end{align*}
In $\logiqueAllen$, $\fcomp$ corresponds to the classical composition
 of relations:
 $\interprel{{\relAQs}\fcomp{\relAQr}}=\interprel{\relAQs}\comp\interprel{\relAQr}$
 (i.e. $\interpretation\models\varx\mathrel{{\relAQs}\fcomp{\relAQr}}\vary$
  if there exists $d\in\Domaine$ such that
  $(\interpretation(\varx), d)\in\interprel{\relAQr}$ and
  $(d, \interpretation(\vary))\in\interprel{\relAQs}$).
In some other QAs, $\fcomp$ corresponds to a different operation,
 called the weak composition~\citep{renz05,ligozat04}.

A formula $\varphi\in\logiqueQA$ is path-consistent if,
   for all $3$ variables $\varx, \vary, \varz\in\Variables$,
    the constraint deduced by composition between $\varx$ and $\varz$
    ($\varx\mathrel{\relDe{\varphi}(\vary, \varz)\fcomp\relDe{\varphi}(\varx, \vary)}\varz$)
    is weaker than the constraint stated in $\varphi$
    (i.e. $\varx\mathrel{\relDe{\varphi}(\varx, \varz)}\varz$).


\subsection{Belief change}

\subsubsection{Belief revision}
 is an operation of belief change.
Intuitively, given the set of beliefs $\psi$ 
 an agent has about a static world, it consists in considering the
 change of their beliefs when faced with a new set of beliefs
 $\mu$, assuming that $\mu$ is considered to be unquestionable
 by the agent.
The resulting set of beliefs is noted $\psi\rev\mu$, and depends on the
 choice of a belief revision operator $\rev$.
In~\citet{agm-85}, the principle of minimal change has been stated
 and could be formulated as follows:
 $\psi$ is minimally changed into $\psi'$ such that the conjunction
 of $\psi'$ and $\mu$ is consistent,
 and the result of the revision is this conjunction.
Hence, there is more than one possible $\rev$ operator, since
 the definition of $\rev$ depends on how belief change is ``measured''.
More precisely, the minimal change principle has been formalized by
 a set of postulates, known as the AGM~postulates---%
 after the names of~\citet*{agm-85}.
\citet{peppa08} presents
 a detailed survey of belief revision at a general level
 (for any formalisms satisfying some general properties, such
  as closure under conjunction)
 including some representation theorems and the discussion of certain related issues
 (other belief change operators, etc.).

In~\citet{katsuno91}, revision has been studied in the framework of
 propositional logic (with a finite set of variables).
The AGM~postulates are translated into this formalism as follows
 ($\psi$, $\psi_1$, $\psi_2$, $\mu$, $\mu_1$, $\mu_2$ and $\phi$ are
 propositional formulas):
\begin{enumerate}[($\rev$1)]
\item
  $\psi\rev\mu\models\mu$.
\item
  If $\psi\land\mu$ is consistent then $\psi\rev\mu\equiv\psi\land\mu$.
\item
  If $\mu$ is consistent then $\psi\rev\mu$ is consistent.
\item
  If $\psi_1\equiv\psi_2$ and $\mu_1\equiv\mu_2$ then
  $\psi_1\rev\mu_1\equiv\psi_2\rev\mu_2$.
\item
  $(\psi\rev\mu)\land\phi \models \psi\rev(\mu\land\phi)$.
\item
  If $(\psi\rev\mu)\land\phi$ is consistent then \\
  $\psi\rev(\mu\land\phi) \models (\psi\rev\mu)\land\phi$.
\end{enumerate}
Moreover, a family of revision operators is defined based on distance functions
 $\dist$ on $\INTERPRETATIONS$, where $\INTERPRETATIONS$ is the set of interpretations:
 the revision of $\psi$ by $\mu$ according to $\revDist$
 ($\psi\revDist\mu$) is such that
 \begin{align}
   \Mod(\psi\revDist\mu)
   &=
   \{\interpretationOmega\in\Mod(\mu)
    ~|~
    \dist(\Mod(\psi), \interpretationOmega) = \distPsiMu\}
    \notag
    \\
    \text{with }\distPsiMu &= \dist(\Mod(\psi), \Mod(\mu))
    \label{eq:revDist}
 \end{align}
Intuitively, $\distPsiMu$ measures, using $\dist$, the minimal modification
 of $\psi$ into $\psi'$ needed to make $\psi'\land\mu$ consistent.

It appears that it is not required for $\dist$ to be a true distance function,
 i.e. symmetry and triangular inequality are not required:
 if $\dist$ verifies the separation postulate, then $\revDist$
 verifies postulates ($\rev$1--6).

This approach can be extended to other formalisms for which a model-theoretic
 semantics can be defined such that a distance function can be defined on
 the set of interpretations $\INTERPRETATIONS$.
However, in some of these formalisms, a representability issue can be raised:
 it may occur that a subset $\SEI$ of $\INTERPRETATIONS$ is not representable,
 i.e. there is no formula $\varphi$ such that $\Mod(\varphi)=\SEI$.
This representability issue is addressed below, for the case of qualitative
 algebras.

Belief revision has been applied to the issue of the adaptation process
 of a case-based reasoning system~\citep{cojan2012,avec_initiales_prenoms_dufourlussier:hal-00871703}.

\subsubsection{Belief contraction}
 is the operation of belief change that associates to a set of beliefs
 $\psi$ and a set of beliefs $\mu$, a set of beliefs $\psi\contraction\mu$
 such that $\psi\contraction\mu\not\models\mu$.
In propositionally closed formalisms, the Harper identity makes it possible to define
 a contraction operator $\contraction$ thanks to a revision operator $\rev$
 with
 \begin{equation}
   \psi\contraction\mu = \psi\lor(\psi\rev\lnot\mu)
   \label{eq:identite-Harper}
 \end{equation}
Conversely, the Levi identity makes it possible to define a revision operator
 $\rev$ with
 \begin{equation*}
   \psi\rev\mu = (\psi\contraction\lnot\mu)\land\mu
 \end{equation*}

\subsubsection{Belief merging}
 is another operation of belief change.
Given some sets of beliefs $\psi_1$, \ldots, $\psi_n$,
 their merging is a set of beliefs
 $\Psi$ that contains ``as much as possible'' of the beliefs in
 the $\psi_i$'s.
Intuitively, $\Psi$ is the conjunction of $\psi_1'$, \ldots, $\psi_n'$
 such that each $\psi_i$ has been minimally modified into $\psi_i'$
 in order to make this conjunction consistent.
Some postulates of belief merging have been proposed and
 discussed~\citep{konieczny02},
 in a similar way as the AGM~postulates.

In practice, studies on belief merging are often easy to
 reuse for belief revision:
 the revision of $\psi$ by $\mu$
 can be seen
 as a kind of merging of $\psi$ and $\mu$ such that no
 modification is allowed on $\mu$.

For instance, belief merging has been studied for
 qualitative algebras by~\citet{condotta10} and \citet{wallgrun10relation}.
\citeauthor{wallgrun10relation} have proposed \emph{syntax-based} revision operators for qualitative algebras.
Those operators do not obey the AGM postulates---%
 most importantly, the syntax-independance postulate.
Therefore, their work cannot serve as a base for developping a model distance-based, AGM revision operator.
\citeauthor{condotta10}, on the other hand, proposed both syntax and semantic-based operators.
The latter can be used as a base to create corresponding revision operators.

\subsection{Belief revision in qualitative algebras}

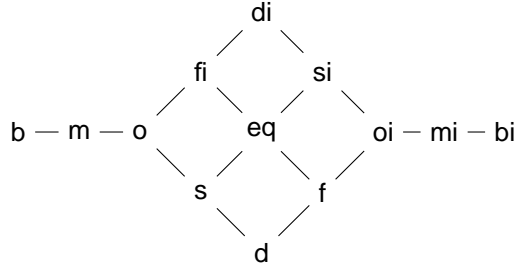
\begin{figure}
  \center{{
  \def\Al#1{\ensuremath{\relQCNsa{#1}}}
  \begin{tikzpicture}[node distance=8mm]
    \node(b)[]{\Al{b}} ;
    \node(m)[right of=b]{\Al{m}} ;
    \node(o)[right of=m]{\Al{o}} ;
    \node(oD)[right of=o]{} ; 
    \node(fi)[above of=oD]{\Al{fi}} ;
    \node(fiH)[above of=fi]{} ; 
    \node(s)[below of=oD]{\Al{s}} ;
    \node(sB)[below of=s]{} ; 
    \node(eq)[right of=oD]{\Al{eq}} ;
    \node(eqD)[right of=eq]{} ; 
    \node(di)[right of=fiH]{\Al{di}} ;
    \node(d)[right of=sB]{\Al{d}} ;
    \node(si)[above of=eqD]{\Al{si}} ;
    \node(f)[below of=eqD]{\Al{f}} ;
    \node(oi)[right of=eqD]{\Al{oi}} ;
    \node(mi)[right of=oi]{\Al{mi}} ;
    \node(bi)[right of=mi]{\Al{bi}} ;
    \draw[-] (b) -- (m) ;
    \draw[-] (m) -- (o) ;
    \draw[-] (o) -- (fi) ;
    \draw[-] (o) -- (s) ;
    \draw[-] (fi) -- (eq) ;
    \draw[-] (s) -- (eq) ;
    \draw[-] (fi) -- (di) ;
    \draw[-] (s) -- (d) ;
    \draw[-] (di) -- (si) ;
    \draw[-] (eq) -- (si) ;
    \draw[-] (d) -- (f) ;
    \draw[-] (eq) -- (f) ;
    \draw[-] (si) -- (oi) ;
    \draw[-] (f) -- (oi) ;
    \draw[-] (oi) -- (mi) ;
    \draw[-] (mi) -- (bi) ;
  \end{tikzpicture}
}}
  \caption{One possible neighborhood graph for the Allen algebra
            \citep{ligozat91generalized}.\label{fig:graphe-voisinage}}
\end{figure}

In~\citet{condotta10} a belief merging operator is defined
 which is based on a distance function $\dist$ on scenarios, defined as follows.
Let $\distRel$ be a distance function on $\RelationsAQBases$.
Let $\scenario, \tau\in\INTERPRETATIONS$, be two scenarios based
 on the same set of variables $\Variables$.
Then, $\dist$ is defined by
 \begin{equation*}
   \dist(\scenario, \tau)
   =
   \sum_{\varx,\vary\in\Variables, \varx\neq\vary}
   \distRel(\relDe{\scenario}(\varx, \vary), \relDe{\tau}(\varx, \vary))
 \end{equation*}
One of the possibilities for $\distRel$ is the use of a neighborhood
 graph, i.e. a connected, undirected graph whose vertices are
 the base relations and such that $\distRel({\relAQr}, {\relAQs})$
 is the length of the shortest path between $\relAQr$ and $\relAQs$.
Figure~\ref{fig:graphe-voisinage} presents such a graph for the Allen
 algebra.
Then, the models of the merging of $\psi_1$, \ldots, $\psi_n$
 is the set of scenarios $\scenario$ that minimizes
 $\sum_{i=1}^n\dist(\Mod(\psi_i), \scenario)$
 (other aggregation functions than the sum can also be used).
The representability issue can be raised since the set of
 the optimal scenarios is not necessarily representable in
 $(\logiqueQA, \models)$.
One solution to address this issue is to find a formula
 $\revision\in\logiqueQA$ whose set of models includes closely
 the set of optimal models.
Another solution is to consider that the result of merging
 is a set of scenarios.

This representability issue is also raised for revision in
 $\logiqueQA$, and the second type of
 solution is used:
 for $\psi, \mu\in\logiqueQA$, $\psi\revDist\mu$ \emph{is}
 the set of the scenarios that are the closest to
 $\Mod(\psi)$.

In~\citet{dufourlussier:hal-00735231}
 and~\citet{avec_initiales_prenoms_dufourlussier:hal-00871703},
 an algorithm for $\revDist$ in a qualitative algebra
 $(\logiqueQA, \models)$ is defined and its implementation in the
 system \revisorQA---for three QAs---is described.
Its inputs are $\psi$ and $\mu$, which are in $\logiqueQA$.
Its output is the set of the scenarios $\scenario\in\Mod(\mu)$
 such that $\dist(\Mod(\psi), \scenario)$ is minimal.
Its principle is based on an A* search~\citep{mmt/Pear84a} with an
 admissible heuristics.
For this search:
 \begin{itemize}
 \item
   A state is a $\varphi\in\logiqueQA$.
 \item
   The initial state is $\mu$.
 \item
   A successor of a state $\varphi$ is a state $\varphi'$
    obtained by substituting in $\varphi$ a constraint
    $\varx\mathrel{{\relAQNum1}\ourel\ldots\ourel{\relAQNum{m}}}\vary$ ($m\geq2$)
    with a (more specific) constraint
    $\varx\relAQNum{k}\vary$ ($1\leq{}k\leq{}m$).
 \item
   A final state is a consistent scenario.
 \item
   The heuristic cost function is an estimation of the distance
    from $\psi$ to the state $\varphi$
    (estimation that is exact on final states).
 \end{itemize}
A slight modification wrt the classical A* algorithm is that the search
 stops after \emph{all} the states at minimal cost have been generated---%
 not as soon as a first final state is found.
The result is the set of final states $\varphi$ which are the models of
 $\mu$ that are the closest to models of $\psi$ according to $\dist$.
It can be noticed that the cost of a final state generated by an A* search
 is $\distPsiMu$ (as defined in~(\ref{eq:revDist})).

The worst-case complexity of this algorithm depends
 on the amount of scenarios in $\mu$,
 which is of the order of
 $O\left(|\RelationsAQBases|^{\frac{|V|\cdot(|V|-1)}{2}}\right)$.

\citet{hue12revising} have also implemented a family of revision operators on QAs.
Their search algorithm is based on the GQR reasoner~\citep{gantner08gqr},
 which does not use a heuristic search
 but, on the other hand, takes advantage of the existence of pre-convex relations---%
 which under certain circumstances make it possible to guarantee consistency
 without having to compute scenarios.

\section{Motivations}

Let us consider the following formulas of $\logiqueAllen$:
 \begin{align*}
   \psi &= \varx\eqAQ\vary \;\land\;
           \vary\eqAQ\varz
   \\
   \mu &= \varx\dAQ\varz  \;\land\; \varz\diAQ\varx
 \end{align*}
The set of models of $\mu$ that are the closest to models of $\psi$
 according to $\dist$ is
 $\SEI=\{\scenarioNum{1}, \scenarioNum{2}, \scenarioNum{3}, \scenarioNum{4}\}$
 with
 \begin{align*}
   \scenarioNum1 &= \mu \;\land\;
                    \varx\dAQ\vary \;\land\;
                    \vary\diAQ\varx \;\land\;
                    \vary\eqAQ\varz \;\land\;
                    \varz\eqAQ\vary
   \\
   \scenarioNum2 &= \mu \;\land\;
                    \varx\sAQ\vary \;\land\;
                    \vary\siAQ\varx \;\land\;
                    \vary\fAQ\varz \;\land\;
                    \varz\fiAQ\vary
   \\
   \scenarioNum3 &= \mu \;\land\;
                    \varx\fAQ\vary \;\land\;
                    \vary\fiAQ\varx \;\land\;
                    \vary\sAQ\varz \;\land\;
                    \varz\siAQ\vary
   \\
   \scenarioNum2 &= \mu \;\land\;
                    \varx\eqAQ\vary \;\land\;
                    \vary\eqAQ\varx \;\land\;
                    \vary\dAQ\varz \;\land\;
                    \varz\diAQ\vary
  \end{align*}
 and it can be proven that no formula $\revision$ of the Allen
 algebra is such that $\Mod(\revision)=\SEI$.\footnote{%
   To prove this, first, let us consider the formula
   \begin{align*}
     \varphi = \mu \;\land\;
               &\varx\mathrel{\dAQ\ourel\sAQ\ourel\fAQ\ourel\eqAQ}\vary \;\land\;
                \vary\mathrel{\diAQ\ourel\siAQ\ourel\fiAQ\ourel\eqAQ}\varx \;\land\; \\
               &\vary\mathrel{\dAQ\ourel\sAQ\ourel\fAQ\ourel\eqAQ}\varz \;\land\;
                \varz\mathrel{\diAQ\ourel\siAQ\ourel\fiAQ\ourel\eqAQ}\vary
   \end{align*}
   $\varphi$ is such that $\SEI\subseteq\Mod(\varphi)$ and
   for each $\chi\in\logiqueAllen$, if $\SEI\subseteq\Mod(\chi)$
   then $\varphi\models\chi$
   ($\varphi$ is the most specific formula whose set of models contains $\SEI$).
   Now, $\SEI\neq\Mod(\varphi)$ since, for instance, the following consistent
    scenario belongs to $\Mod(\varphi)$ and not to $\SEI$:
    \begin{equation*}
      \scenario = \mu \;\land\;
                  \varx\dAQ\vary \;\land\;
                  \vary\diAQ\varx \;\land\;
                  \vary\dAQ\varz \;\land\;
                  \varz\dAQ\vary
    \end{equation*}
   Therefore, there is no $\revision\in\logiqueAllen$ such that
    $\Mod(\revision)=\SEI$.}
So, the representability issue is raised:
 $\revDist$ in $\logiqueAllen$ does not match exactly equation~(\ref{eq:revDist}).
Thus,
 either $\psi\revDist\mu$ gives a result outside of $\logiqueAllen$
 or $\psi\revDist\mu$ gives a formula $\revision$ that approximates
 the equality~(\ref{eq:revDist}).
By contrast, $\revDist$ defined by this equality can be defined in
 the propositional closure of the Allen algebra
 (which is a consequence of proposition~\ref{prop:representabilite},
  given in the next section),
 and this gives a first motivation for this work.

The second motivation is linked to the expressiveness of the formalisms:
 some knowledge are more easily represented in the propositional closure
 of a QA.
An example will be presented that is formalized using both $\logiqueAllen$
 and its propositional closure $\logiqueAllenCP$.
It appears to be much simpler (or ``more natural'') to formalize it
 in $\logiqueAllenCP$.
Moreover, still on this particular example, the computing
 time of the revision is shorter in the more expressive formalism
 $\logiqueAllenCP$, with the systems we have implemented.

The third motivation of this work is that a revision
 operator on the propositional closure of a QA can be used in the definition
 of a contraction operator, thanks to~(\ref{eq:identite-Harper}),
 which requires disjunction and negation connectors.

\section{Propositional closure of a qualitative algebra}

Let $(\logiqueQA, \models)$ be a qualitative algebra.
The propositional closure of this formalism, as defined in the preliminaries,
 is $(\logiqueQACP, \modelsCP)$.

\begin{proposition}[representability]\label{prop:representabilite}
  Every set of scenarios $\SEI\subseteq\INTERPRETATIONS$
   is representable in $\logiqueQACP$.
  More precisely, with $\displaystyle\varphi=\bigvee_{\scenario\in\SEI}\scenario$,
   $\Mod(\varphi)=\SEI$.
\end{proposition}
\begin{proof}
First, it is proven that
 \begin{equation}
   \text{for any $\scenario\in\INTERPRETATIONS$,}\quad
   \Mod(\scenario)=\{\scenario\}
   \label{eq:modele-scenario}
 \end{equation} 
$\scenario\in\Mod(\scenario)$ is a direct consequence of
 $\scenario\models\scenario$, thus it is sufficient to prove
 that each $\scenarioTau\in\INTERPRETATIONS$ such that $\scenarioTau\neq\scenario$
 is not a model of $\scenario$.
$\scenarioTau\neq\scenario$ implies that there exists
 $\varx, \vary\in\Variables$ and
 ${\relAQ}, {\relAQs}\in\RelationsAQBases$ with $\relAQ\neq\relAQs$
 such that $\varx\relAQ\vary$ and $\varx\relAQs\vary$ are
 respectively a constraint of $\scenario$ and of $\scenarioTau$.
Since $\interprel{\relAQ}\cap\interprel{\relAQs}=\emptyset$
 ($\interprel{\RelationsAQBases}$ being a partition of $\INTERPRETATIONS$)
 and $\scenarioTau\models\varx\relAQs\vary$,
 $\scenarioTau\not\models\varx\relAQ\vary$,
 and therefore, $\scenarioTau\not\models\scenario$,
 which proves~(\ref{eq:modele-scenario}).

From~(\ref{eq:modele-scenario}) and the semantics of $\lor$,
 it comes that
 $\Mod(\varphi)=\bigcup_{\scenario\in\SEI}\Mod(\scenario)=\bigcup_{\scenario\in\SEI}\{\scenario\} \linebreak =\SEI$,
 which proves the proposition.
\end{proof}

Every formula of $\logiqueQACP$ can be written in DNF, since
 it is a propositionally closed formalism, but the following
 proposition goes beyond that.

\begin{proposition}[normal forms]\label{prop:formesNormales}
  Let $\varphi\in\logiqueQACP$.
  $\varphi$ can be put under the following forms:
  \begin{description}
  \item[\DNFWoN form]
    $\varphi$ is equivalent to a formula in DNF using no
     negation symbol.
   \item[\DNFWoNBR form]
     $\varphi$ is equivalent to a formula in DNF using no
     negation symbol and such that its constraints contain
     only base relations.
  \end{description}
\end{proposition}
\begin{proof}
\subproof{\DNFWoN form.}
Let $\varphi_1$ be a formula under DNF equivalent to $\varphi$
 (it exists: cf. the section on preliminaries).
Therefore $\varphi_1$ has the form
 $\varphi_1=\bigvee_i\bigwedge_j\ell_{ij}$
 where $\ell_{ij}$ is either a constraint (positive literal)
 or the negation of a constraint (negative literal).

Let $\lnot(\varx\relAQ\vary)$ be a negative literal.
Let $R$ be the set of base relations occurring in $\relAQ$
 (if ${\relAQ}={\relAQNum1}\ourel\ldots\ourel{\relAQNum{m}}$
  then $R=\{{\relAQNum1}, \ldots, {\relAQNum{m}}\}$)
 and $\overline{R}=\RelationsAQBases\setminus{}R$.
Let $\relAQs$ be the relation based on the relations of $\overline{R}$.
Then, it comes that:
 \begin{equation*}
   \lnot(\varx\relAQ\vary)
   \quad\equiv\quad
   \varx\relAQs\vary
 \end{equation*}
 (for example, $\lnot(\varx\toutAQ\vary)\equiv\varx\aucunAQ\vary$).
Therefore every negative literal can be substituted by an equivalent
 positive literal and, by doing such substitutions on $\varphi_1$,
 the result is a formula $\varphi_2$, equivalent to $\varphi$,
 which proves that $\varphi$ can be put under \DNFWoN form.

\subproof{\DNFWoNBR form.}
First, it is proven that any constraint $\varx\relAQ\vary$
 is equivalent to a formula containing constraints based only
 on base relations (i.e. no occurrence of the symbol $\ourel$).
If ${\relAQ}={\aucunAQ}$, then $\varx\relAQ\vary$ is an inconsistent
 formula and therefore is equivalent to any
 inconsistent formula, for example
 $\varx\relAQ\vary\;\land\;\varx\relAQs\vary$
 (${\relAQ}, {\relAQs}\in\RelationsAQBases$, $\relAQ\neq\relAQs$),
 which is only based on base relations.
If ${\relAQ}\neq{\aucunAQ}$
 then ${\relAQ}={\relAQNum1}\ourel\ldots\ourel{\relAQNum{m}}$
 with $m\geq1$ and then
 \begin{equation*}
   \varx \relAQ \vary
   \quad\equiv\quad
   \varx \relAQNum1 \vary
   \;\lor\ldots\lor\;
   \varx \relAQNum{m} \vary
 \end{equation*}

Second, let $\varphi_2$ be a formula equivalent to $\varphi$ that
 is under \DNFWoN form.
By substituting in $\varphi_2$ all the constraints by equivalent
 formulas based only on base relations, the resulting formula, $\varphi_3$,
 is equivalent to $\varphi$, and contains only base relations and no
 negation.
Finally, $\varphi_3$ can be put under DNF as explained in the
 preliminaries of the paper
 (i.e. according to the set of equivalences~(\ref{eq:equivalences-pour-DNF}))
 resulting in a formula $\varphi_4$
 that is under \DNFWoNBR and which is equivalent to $\varphi$.
\end{proof}

Other authors as well stressed the interest of being able to handle 
 temporal constraints disjunctions, such as
 ``the trip takes either $5$ minutes (by car) or $15$ minutes (by bus).''
These disjunctions are generally not taken into account in the existing
 representations of qualitative relational algebras.
Some work proposed to handle disjunctions in the point algebra \citep{ViKa86}.
In \citet{gerivini95}, for instance,
 qualitative relations between intervals
 are represented by disjunctions of relations between the ends of the
 intervals---%
 e.g. ``the beginning of interval $y$ is before the beginning of interval $x$ or
  the end of $x$ is before the beginning of $y$.''
Formalisms representing temporal metric constraints are more frequent,
 following the proposition of \citet{dechter91temporal}.
In \citet{barber2000},
 disjunctions of constraints are handled using a notion of temporal context.
As far as we know, 
 none of these works has addressed the issue of propositional closure, though.

\section{Belief revision in $(\logiqueQACP, \modelsCP)$}

Given a distance function $\dist$ on $\INTERPRETATIONS$,
 a revision operator on $(\logiqueQACP, \modelsCP)$ can be
 defined according to equation~(\ref{eq:revDist}).
Indeed, proposition~\ref{prop:representabilite} implies
 that
 $\{\interpretationOmega\in\Mod(\mu)~|~\dist(\Mod(\psi), \interpretationOmega) \linebreak = \distPsiMu\}$
 is representable.

\subsection{An algorithm for computing $\revDist$ in $\logiqueQACP$}

The principle of the algorithm is based on the following proposition.

\begin{proposition}[revision of disjunctions]\label{prop:revision-disjonctions}
  Let $\psi$ and $\mu$ be two formulas of $\logiqueQACP$ and
   $\{\psi_i\}_i$ and $\{\mu_j\}_j$ be two finite
   families of $\logiqueQACP$
   such that
   $\displaystyle\psi=\bigvee_i\psi_i$ and
   $\displaystyle\mu=\bigvee_j\mu_j$.

  Let 
   $\distPsiMu_{ij}=\dist(\Mod(\psi_i), \Mod(\mu_j))$
   for any $i$ and $j$.
  Then:
  \begin{align}
    \psi\revDist\mu
    &\equiv
    \bigvee_{i, j, \distPsiMu_{ij}=\distPsiMu}
    \psi_i\revDist\mu_j
    \notag
    \\
    \text{with }
    \distPsiMu
    &=
    \dist(\Mod(\psi), \Mod(\mu))
    \notag
    \\
    \text{Moreover,}\quad
    \distPsiMu
    &=
    \min_{ij}\distPsiMu_{ij}
    \label{eq:Delta=minDeltaij}
  \end{align}
\end{proposition}
\begin{proof}
First,~(\ref{eq:Delta=minDeltaij}) is proven:
 \begin{align*}
   \distPsiMu
   &=
   \dist(\Mod(\psi), \Mod(\mu))
   =
   \distPar{\bigcup_i\Mod(\psi_i), \bigcup_j\Mod(\mu_j)}
   \\
   &=
   \min_{ij}\dist(\Mod(\psi_i), \Mod(\mu_j))
   =
   \min_{ij}\distPsiMu_{ij}
 \end{align*}

Second, let $\interpretationOmega\in\Mod(\psi\revDist\mu)$.
Thus, there exists $\interpretationNu\in\Mod(\psi)$ such that
 $\dist(\interpretationNu, \interpretationOmega)=\distPsiMu$.
Let $i$ and $j$ be such that $\interpretationNu\in\Mod(\psi_i)$ and
 $\interpretationOmega\in\Mod(\mu_j)$.
So, the following chain of relations holds:
 \begin{equation*}
   \distPsiMu
   =
   \dist(\interpretationNu, \interpretationOmega)
   \geq
   \dist(\Mod(\psi_i), \interpretationOmega)
   \geq
   \distPsiMu_{ij}
   \geq
   \distPsiMu
 \end{equation*}
Therefore, all the numbers in this chain are equal
 and
 $\dist(\Mod(\psi_i), \interpretationOmega)=\distPsiMu_{ij}=\distPsiMu$,
 so $\interpretationOmega\in\Mod(\psi_i\revDist\mu_j)$
 for $i$ and $j$, such that $\distPsiMu_{ij}=\distPsiMu$.
To summarize, if $\interpretationOmega\in\Mod(\psi\revDist\mu)$
 then
 $\interpretationOmega\in\ModPar{\bigvee_{i, j, \distPsiMu_{ij}=\distPsiMu}\psi_i\revDist\mu_j}$.

Conversely, let $\interpretationOmega\in\Mod(\psi_i\revDist\mu_j)$
 for $i$ and $j$ such that $\distPsiMu_{ij}=\distPsiMu$.
This entails that $\dist(\Mod(\psi_i), \interpretationOmega)=\distPsiMu$,
 hence the following chain of relations:
 \begin{equation*}
   \distPsiMu
   \leq
   \dist(\Mod(\psi), \interpretationOmega)
   \leq
   \dist(\Mod(\psi_i), \interpretationOmega)
   =
   \distPsiMu
 \end{equation*}
 so $\dist(\Mod(\psi), \interpretationOmega)=\distPsiMu$ with
 $\interpretationOmega\in\Mod(\mu)$,
 consequently
 $\interpretationOmega\in\Mod(\psi\revDist\mu)$.

To conclude,
 $\interpretationOmega\in\Mod(\psi\revDist\mu)$
 iff
 $\interpretationOmega\in\ModPar{\bigvee_{i, j, \distPsiMu_{ij}=\distPsiMu}\psi_i\revDist\mu_j}$,
 which proves the proposition.
\end{proof}

The algorithm for $\revDist$ in $\logiqueQACP$ consists
 roughly in putting $\psi$ and $\mu$ in $\DNFWoN$ form
 then applying proposition~\ref{prop:revision-disjonctions}
 on them, using the $\revDist$ algorithm on $\logiqueQA$
 for computing the $\psi_i\revDist\mu_j$'s.

This requires some small modifications in the algorithm
 for $\revDist$ in $(\logiqueQA, \models)$:
 \begin{itemize}
 \item
   The revision algorithm inputs a triple $(\psi, \mu, \distmax)$
    where $\psi, \mu\in\logiqueQA$ and $\distmax$ is a non negative
    number which gives a maximal
    admissible value for $\distPsiMu=\dist(\Mod(\psi), \Mod(\mu))$.
  \item
    The search in the state space is stopped
     (and returns a ``failure symbol'') when the
     cost associated to a state is greater than $\distmax$.
  \item
    The output of the algorithm is either the failure symbol
     or a pair $(\revision, \distPsiMu)$
     where $\revision\in\logiqueQACP$ is the disjunction of scenarios
     of $\psi\revDist\mu$.
 \end{itemize}

\def\alors{\textbf{then}}
\def\commentaire#1{\strut\hfill\emph{// #1}}
\def\echec{\fm{failure}}
\def\faire{\textbf{do}}
\def\finsi{\textbf{end if}}
\def\finpour{\textbf{end for}}
\def\IND{~~~~} 
\def\pourchaque{\textbf{for each}\xspace}
\def\etchaque{\textbf{and each}\xspace}
\def\return{\textbf{return}\xspace}
\def\si{\textbf{if}\xspace}
\def\sinonsi{\textbf{else if}\xspace}
\def\resultat{\fm{result}}
\def\unerevision{\fm{rev}}
\begin{figure}
  $\RevisionQACP(\psi, \mu)$
  \begin{description}
  \item[input]
    $\psi, \mu\in\logiqueQACP$
  \item[output]
    $\revision\in\logiqueQACP$ such that $\revision\equiv\psi\revDist\mu$
  \end{description}
  \begin{enumerate}[\small1\normalsize]
  \item
    $\psi\aff\DNFWoN(\psi)$ \hfill$\psi=\bigvee_i\psi_i$ where $\psi_i\in\logiqueQA$
  \item
    $\mu\aff\DNFWoN(\mu)$   \hfill$\mu=\bigvee_j\mu_j$ where $\mu_j\in\logiqueQA$
  \item
    $\resultat\aff\emptyset$
  \item
    $\distmax\aff+\infty$
  \item
    \pourchaque $i$ \etchaque $j$ \faire
  \item
    \IND$\unerevision\aff\RevisionQA(\psi_i, \mu_j, \distmax)$
  \item
    \IND\si $\unerevision\neq\echec$ \alors
  \item
    \IND\IND$(\revision_{ij}, \distPsiMu_{ij})\aff\unerevision$
            \commentaire{$\revision_{ij}=\psi_i\revDist\mu_j\in\logiqueQACP$}\\
            \commentaire{$\distPsiMu_{ij}=\dist(\Mod(\psi_i), \Mod(\mu_j))$}
  \item
    \IND\IND\si $\distPsiMu_{ij}<\distmax$ \alors
  \item
    \IND\IND\IND$\distmax\aff\distPsiMu_{ij}$
  \item
    \IND\IND\IND$\resultat\aff\{\revision_{ij}\}$
  \item
    \IND\IND\sinonsi $\distPsiMu_{ij}=\distmax$ \alors
  \item
    \IND\IND\IND$\resultat\aff\resultat\cup\{\revision_{ij}\}$
  \item
    \IND\IND\finsi
  \item
    \IND\finsi
  \item
    \finpour
  \item
    $\revision\aff\bigvee_{\scenario\in\resultat}\scenario$
  \item
    \return $\revision$
  \end{enumerate}
\caption{Algorithm for $\revDist$ in $(\logiqueQACP, \modelsCP)$.\label{algo:rev-QACP}}
\end{figure}

The algorithm is shown in figure~\ref{algo:rev-QACP}. It
 is based on the proposition~\ref{prop:revision-disjonctions}
 and on the modified algorithm for $\revDist$ in
 $(\logiqueQA, \models)$---%
 line~6 makes use of this modified algorithm.

\subsection{\revisorPCQA: an implementation of $\revDist$ in $(\logiqueQACP, \modelsCP)$}

\paragraph{\revisor}
 is a collection of several revision engines that are open-source and
 freely available.\footnote{\url{http://revisor.loria.fr}}

In particular, \revisorQA implements $\revDist$ in three QAs:
 the Allen algebra,
 INDU\linebreak---an extension of the Allen algebra taking into account relations between intervals
 according to their lengths~\citep{pujari99indu}---and
 RCC8---a QA for representing topological relations between regions of
 space~\citep{randell92spatial}.
Moreover, it is easy to use a different qualitative algebra,
 by specifying in the code the value of ${\relAQs}\fcomp{\relAQr}$ for
 each ${\relAQr}, {\relAQs}\in\RelationsAQBases$,
 the value of $\inverserel{\relAQ}$ for each ${\relAQ}\in\RelationsAQBases$,
 and the neighborhood graph.
The engine is written in Perl,
 but can be used through a Java library.
The worst-case complexity of this implementation is of the order of
 $O\left(|\RelationsAQBases|^{\frac{|\Variables|\cdot(|\Variables|-1)}{2}}\right)$.

\paragraph{\revisorPCQA}
 implements $\revDist$ on the propositional closures of the QAs $\logiqueAllen$,
 INDU and RCC8:
 it actually uses \revisorQA
 and is one of the engines of \revisor.
The worst-case complexity of this implementation is of the order of \linebreak
 $O\left(|\Variables|^4|\RelationsAQBases|^{\frac{|\Variables|\cdot(|\Variables|-1)}{2}}\right)$,
 according to a coarse analysis.

\subsubsection{Examples}
%
The following examples have been executed using \revisorPCQA,
 and are included with the source code.
The \verb=README= file associated with \revisorQA on the \revisor
 website explains how they can be executed.

\def\Cours{\fm{Courses}}
\def\Periodes{\fm{Periods}}
\def\anglais{\fm{English}}
\def\biologie{\fm{biology}}
\def\histoire{\fm{history}}
\def\mathematiques{\fm{maths}}
\def\huitneuf{\fm{8-9}}
\def\neufdix{\fm{9-10}}
\def\dixonze{\fm{10-11}}
\def\onzedouze{\fm{11-12}}
\def\huitdouze{\fm{8-12}}
\def\Zoe{\text{Zo{\'e}}}
%
\subparagraph{The first example}
%
 aims at showing that some revision problems are more easily
 expressed in $\logiqueQACP$ than in $\logiqueQA$.
Let us consider $\Zoe$, a school principal that has to schedule a morning with $4$
 courses in biology, English, history and maths for a group of students.
For this purpose, she plans to reuse the previous year schedule:
 \begin{align*}
   \pi
   &=
   \anglais\eqAQ\huitneuf
   \;\land\;
   \biologie\eqAQ\neufdix
   \\
   \;&\land\;
   \histoire\eqAQ\dixonze
   \;\land\;
   \mathematiques\eqAQ\onzedouze
 \end{align*}
 stating, e.g., that the English course takes place from $8$ to $9$ a.m.

She also has some background knowledge that she expresses first
 in $\logiqueAllen$.
She knows the relation between the $4$ time periods:
 \begin{align*}
   \beta_1
   =
   \huitneuf\mAQ\neufdix
   \;\land\;
   \neufdix\mAQ\dixonze
   \;\land\;
   \dixonze\mAQ\onzedouze
 \end{align*}
Then, she states that every course $c_1$ has no intersection
 (except, possibly, on one of the boundaries) with another course $c_2$:
 \begin{align*}
   \beta_2
   &=
   \bigwedge_{c_1, c_2\in\Cours, c_1\neq{}c_2}
   c_1 \mathrel{{\bAQ}\ourel{\biAQ}\ourel{\mAQ}\ourel{\miAQ}} c_2
   \\
   \text{with }
   \Cours
   &=
   \{\biologie, \anglais, \histoire, \mathematiques\}
 \end{align*}
Then, she aims at representing the fact that each course corresponds to
 one of the $4$ time periods.
Since there is no disjunction in $\logiqueAllen$, she uses the following
 trick:
 asserting that each course is either equal or has no intersection
 (except on the boundaries) with any period:
 \begin{align*}
   \beta_3
   &=
   \bigwedge_{c\in\Cours, p\in\Periodes}
   c \mathrel{{\eqAQ}\ourel{\bAQ}\ourel{\biAQ}\ourel{\mAQ}\ourel{\miAQ}} p
   \\
   \text{with }
   \Periodes
   &=
   \{\huitneuf, \neufdix, \dixonze, \onzedouze\}
 \end{align*}
In order to prevent the courses and the periods to exceed the
 boundaries of the morning, the variable $\huitdouze$ is introduced
 and the following knowledge about it is asserted:
 \begin{align*}
   \beta_4
   &=
   \huitneuf\sAQ\huitdouze
   \;\land\;
   \neufdix\dAQ\huitdouze
   \;\land\;
   \dixonze\dAQ\huitdouze \\
   \;&\land\;
   \onzedouze\fAQ\huitdouze
   \;\land\;
   \bigwedge_{c\in\Cours} c \mathrel{{\sAQ}\ourel{\dAQ}\ourel{\fAQ}} \huitdouze
 \end{align*}
Let $\beta{}=\beta_1\land\beta_2\land\beta_3\land\beta_4$.
Then, the knowledge about the previous year is $\psi=\beta{}\land\pi$.
For the current year, a new constraint is that the biology and history
 teachers should not meet (for some reason):
 \begin{align*}
   \gamma
   &=
   \biologie \mathrel{{\bAQ}\ourel{\biAQ}} \histoire
 \end{align*}
Since the background knowledge has not changed, the knowledge about this
 year is $\mu=\beta{}\land\gamma$.
Thus, to propose a new schedule, $\Zoe$ will revise $\psi$ by $\mu$.
If she uses the $\revDist$ revision operator defined above, there are two
 models that consist in switching English with biology or
 history with maths.

Now, $\Zoe$ wants to formalize its knowledge in $\logiqueAllenCP$.
The previous year schedule $\pi$ and the new constraint $\gamma$ for the
 current years are kept.
What changes is the representation of background knowledge:
 $\CP{\beta}=\beta_1\land\CP{\beta_2}\land\CP{\beta_3}$,
 with $\CP{\beta_2}$ expressing the fact that two courses cannot
 occur in the same period of time
 \begin{align*}
   \CP{\beta_2}
   =
   \bigwedge_{c_1, c_2\in\Cours, c_1\neq{}c_2} \lnot(c_1\eqAQ{}c_2)
 \end{align*}
 and $\CP{\beta_3}$ expressing the fact that each course is in one
 of the $4$ periods:
 \begin{align*}
   \CP{\beta_3}
   =
   \bigwedge_{c\in\Cours}\bigvee_{p\in\Periodes} c \eqAQ p
 \end{align*}
The revision of $\CP{\psi}=\CP{\beta}\land\pi$ by
 $\CP{\mu}=\CP{\beta}\land\gamma$ also gives two models,
 corresponding to the two same course exchanges.
(Formally, they are not the same models, since the sets of
 variables are different---there is an additional variable
 in the first formalization: $\huitdouze$.)

Our claim is that the second formalization is simpler than the first
 one, which has required a ``trick''.
Furthermore, \revisorQA requires about $6$ minutes to solve
 this problem (in the first formalization)
 whereas \revisorPCQA only requires about $2$ minutes.

\def\nr{---} 
\begin{table*}[bt]
\begin{center}
 \tabcolsep=0.1cm
 \begin{tabular}{|c|c|c|c|c|c|c|c|}
 \cline{3-8}
 \multicolumn{2}{c|}{\strut}  & \multicolumn{3}{c|}{\revisorQA} & \multicolumn{3}{c|}{\revisorPCQA} \\
  \hline
 $n$ & $p$ & \#Variables & Avg distance & Avg time (s) & \#Variables & Avg distance & Avg time (s)  \\  \hline
 3   & 0 & 7  & 24.0 & 1.387    & 6  & 22.0 & 3.809    \\  \hline
 3   & 1 & 8  & 21.0 & 5.407    & 6  & 20.0 & 7.744    \\  \hline
 4   & 0 & 9  & 25.3 & 444.927  & 8  & 24.7 & 119.136  \\  \hline
 4   & 1 & 10 & 29.3 & 765.125  & 8  & 29.3 & 183.945  \\  \hline
 4   & 2 & 11 & 14.0 & 2040.551 & 8  & 14.0 & 266.667  \\  \hline
 5   & 0 & 11 & \nr  & \emph{$>$ 1 hour} & 10 & 26.0 & 3052.398 \\  \hline
 5   & 1 & 12 & \nr  & \emph{$>$ 1 hour} & 10 & \nr  & \emph{$>$ 1 hour} \\  \hline
 \end{tabular}
 \end{center}
 \caption{%
   Average distance and average time according to the problem, parametrized
    by $n$ and $p$, where $n$ is the number of courses and of time periods and
    $p$ is the number of breaks during the global time period.
   ``Avg distance'' is the average of the $\distPsiMu$ values on the set of
     revision problems generated for a given pair $(n, p)$.
   \label{tab-time}}
\end{table*}
\subparagraph{The second example}
%
 generalizes the first one.
It consists in a family of examples parametrized
 by $n$ and $p$, where $n$ is the number of courses and of
 time periods (the first example corresponds to $n=4$) and 
 $p$ is the number of breaks during the global time period 
 (the first example corresponds to $p=0$).
Moreover, the breaks in the examples are uniformally spread
 throughout the whole period.
It has been experimented with $n\in\{3, 4, 5\}$ and 
$p\in\{0, 1, 2\}$. Following similar formalizations in $\logiqueAllen$
 (with $2n+p+1$ variables) and $\logiqueAllenCP$
 (with $2n$ variables), the result were the same 
 (except for the additional variables)
 and the computing times are presented in table~\ref{tab-time}.
 The average time, for each line, is computed with series of $n-1$ tests with $d>0$, 
 on a computer with a \mes{2.53}{GHz} processor and \mes{8}{GB} of available memory. 
 For example, for $n=4$ and $p=1$, 
 the average distance is $14.0$ for \revisorQA and \revisorPCQA 
 and the average time is \mes{765.125}{s} for \revisorQA and \mes{183.945}{s} for \revisorPCQA. 
 The average time increases with the number of variables for \revisorQA and 
 for \revisorPCQA. For the same number of variables, \revisorQA is faster than \linebreak
 \revisorPCQA. However, as fewer additional variables are introduced under \linebreak
 \revisorPCQA, more complex problems can be solved with \revisorPCQA than 
 with \revisorQA.

\def\Boole{\fm{Boole}}
\def\DeMorgan{\fm{De~Morgan}}
\def\Weierstrass{\fm{Weierstra\ss}}
%
%
\subparagraph{The third example}
 uses a belief contraction operator.
As stated by equation~(\ref{eq:identite-Harper}), a contraction operator
 can be defined based on the revision operator $\revDist$.
Let $\contractionDist$ be this operator.
Now let us consider the set of beliefs $\psi$ of an agent called Maurice
 about the dates of birth and death of famous mathematicians.
Maurice thought that Boole was born after de~Morgan and died before
 him and that de~Morgan and Weierstra{\ss} were born the same year
 (say, at the same time) but the former died before the latter:
 \begin{align*}
   \psi
   =
   \Boole\dAQ\DeMorgan
   \;\land\;
   \DeMorgan\sAQ\Weierstrass
 \end{align*}
 where, $\Boole$ is the interval of time between the birth and
 the death of Boole, and so on.
Now, Germaine, a friend of Maurice, tells him that she is not sure whether
 Boole was born strictly after Weierstra{\ss}.
Since Maurice trusts Germaine (and her doubts), he wants to make the contraction of
 its original beliefs $\psi$ by $\mu$ with
 \begin{align*}
   \mu
   =
   \Boole
   \mathrel{{\biAQ} \ourel {\miAQ} \ourel {\oiAQ} \ourel {\fAQ} \ourel {\dAQ}}
   \Weierstrass
 \end{align*}
The result, computed by \revisorPCQA in less than one second, is
 $\psi\contractionDist\mu$, equivalent to the following formula:
 \begin{align*}
   &\strut\mathrel{\text{\phantom{$\lor$}}}
   (\Boole\dAQ\DeMorgan
    \;\land\;
    \DeMorgan\sAQ\Weierstrass)
    \\
   &\strut\lor
    (\Boole\sAQ\Weierstrass
     \;\land\;
     \DeMorgan\diAQ\Weierstrass)
     \\
   &\strut\lor
    (\Boole\sAQ\DeMorgan
     \;\land\;
     \DeMorgan\sAQ\Weierstrass
     )
     \\
   &\strut\lor
    \left(\!\!\!\!\mlc{$\Boole\dAQ\DeMorgan
                        \;\land\;
                        \Boole\sAQ\Weierstrass$ \\
                       $\;\land\;
                        \DeMorgan\oAQ\Weierstrass$}\!\!\!\!\right)
 \end{align*}
Actually, the last term of this disjunction corresponds to the
 reality, provided that the intervals of time correspond to a
 year granularity.\footnote{%
   George Boole (1815-1864),
   Augustus De~Morgan (1806-1871),
   Karl Weierstra{\ss} (1815-1897).}

\section{Conclusion}

This paper has presented an algorithm for distance-based belief revision in the
 propositional closure $\logiqueQACP$ of a qualitative algebra $\logiqueQA$,
 using the revision operation on $\logiqueQA$.
This work is motivated by the fact that it gives a revision operation whose
 result is representable in the formalism,
 by the fact that some practical examples are easily represented in $\logiqueQACP$
 whereas they are quite difficult to represent in $\logiqueQA$,
 and by the fact that it makes it possible to define a contraction operator thanks to the
 Harper identity (which requires disjunction and negation).
The preprocessing of the algorithm consists in putting the formulas into
 a disjunctive normal form without negation.
Then, proposition~\ref{prop:revision-disjonctions},
 which reduces a revision of disjunctions to a disjunction
 of the least costly revisions, is applied.
\revisorPCQA is an implementation of this revision operator for
 the Allen algebra, INDU and RCC8.

A first direction of research following this work is the improvement of the
 computation time of the \revisorPCQA system.
One way to do it is to parallelize it, which should not be very difficult
 (parallelizing the main loop).
A sequential optimization would consist in finding a heuristic for
 ranking the pairs $(i, j)$, with the aim of starting from the best candidates,
 in order to obtain a low upper bound $\distmax$ sooner.

The approach depicted in this paper for an algorithm of $\revDist$ in
 $\logiqueQACP$ built using an algorithm of $\revDist$ in $\logiqueQA$
 has actually little dependence on the peculiarities of QAs
 (except for the fact that negations can be removed in $\logiqueQACP$
  according to proposition~\ref{prop:formesNormales}).
Indeed, it could be reused as such for designing an algorithm
 of a revision on the disjunctive closure of a formalism $\logique$,
 provided that an algorithm of $\revDist$ has been designed in $\logique$.
For example, the \revisorCLC system has been implemented in the formalism
 $\logiqueCLC$ of conjunction of linear constraints (on integers and real
 numbers), with a city block distance~\citep{DBLP:conf/ewcbr/CojanL08}.
However, reusing this approach for having an algorithm of $\revDist$ in a
 \emph{propositional} closure $\logiqueCP$ raises additional issues.
In particular, the minimal distance between sets of models
 (i.e. $n$-tuples of numbers) is not necessarily reached, thus
 violating the postulate ($\rev$3).
Working on this issue is a second direction of research.

This paper has described an algorithm for belief revision in $\logiqueQACP$,
 which can be straightforwardly used for belief contraction.
The third direction of research is to study how other belief change
 operations can be implemented in this formalism,
 in particular belief merging~\citep{konieczny02}
 and knowledge update~\citep{katsuno-mendelzon:1991a}.

\section*{Acknowledgments}

The authors would like to express their sincere gratitude to the reviewers
 of the version of this article that was submitted to KR~2014.
Those of their suggestions which were not addressed in the KR version of this article for want of space
 are addressed in this technical report.
 
This research was partially funded by the project Kolflow\footnote{%
  \url{http://kolflow.univ-nantes.fr}}
 of the French National Agency for Research (ANR), program ANR CONTINT.

\bibliography{biblio}

\begin{thebibliography}{26}
\providecommand{\natexlab}[1]{#1}
\providecommand{\url}[1]{\texttt{#1}}
\expandafter\ifx\csname urlstyle\endcsname\relax
  \providecommand{\doi}[1]{doi: #1}\else
  \providecommand{\doi}{doi: \begingroup \urlstyle{rm}\Url}\fi

\bibitem[Alchourr{\'{o}}n et~al.(1985)Alchourr{\'{o}}n, G{\"{a}}rdenfors, and
  Makinson]{agm-85}
C.~E. Alchourr{\'{o}}n, P.~G{\"{a}}rdenfors, and D.~Makinson.
\newblock {On the Logic of Theory Change: partial meet functions for
  contraction and revision}.
\newblock \emph{{Journal of Symbolic Logic}}, 50:\penalty0 510--530, 1985.

\bibitem[Allen(1983)]{allen83cacm}
J.~F. Allen.
\newblock {Maintaining knowledge about temporal intervals}.
\newblock \emph{{Communications of the ACM}}, 26\penalty0 (11):\penalty0
  832--843, November 1983.

\bibitem[Barber(2000)]{barber2000}
F.~Barber.
\newblock Reasoning on interval and point-based disjunctive metric constraints
  in temporal contexts.
\newblock \emph{Journal of Artificial Intelligence Research}, 12\penalty0
  (2000):\penalty0 35--86, 2000.

\bibitem[Cojan and Lieber(2008)]{DBLP:conf/ewcbr/CojanL08}
J.~Cojan and J.~Lieber.
\newblock {Conservative Adaptation in Metric Spaces}.
\newblock In \emph{Advances in Case-Based Reasoning, 9th European Conference,
  ECCBR-2008, Trier, Germany. Proceedings}, LNAI 5239, pages 135--149, 2008.

\bibitem[Cojan and Lieber(2012)]{cojan2012}
J.~Cojan and J.~Lieber.
\newblock {Belief revision-based case-based reasoning}.
\newblock In G.~Richard, editor, \emph{{Proceedings of the ECAI-2012 Workshop
  SAMAI: Similarity and Analogy-based Methods in AI}}, pages 33--39, 2012.

\bibitem[Condotta et~al.(2010)Condotta, Kaci, Marquis, and Schwind]{condotta10}
J.-F. Condotta, S.~Kaci, P.~Marquis, and N.~Schwind.
\newblock {A Syntactical Approach to Qualitative Constraint Networks Merging}.
\newblock In \emph{{Proc. of the 17th LPAR (Logic for Programming, Artificial
  Intelligence and Reasoning)}}, pages 233--247, 2010.

\bibitem[Dechter et~al.(1991)Dechter, Meiri, and Pearl]{dechter91temporal}
R.~Dechter, I.~Meiri, and J.~Pearl.
\newblock {Temporal constraint networks}.
\newblock \emph{Artificial Intelligence}, 49:\penalty0 61--95, 1991.

\bibitem[Dufour-Lussier et~al.(2012)Dufour-Lussier, {Le Ber}, Lieber, and
  Martin]{dufourlussier:hal-00735231}
V.~Dufour-Lussier, F.~{Le Ber}, J.~Lieber, and L.~Martin.
\newblock {Adapting Spatial and Temporal Cases}.
\newblock In I.~Watson and B.~{D\'iaz Agudo}, editors, \emph{{ICCBR}}, volume
  7466 of \emph{LNAI}, pages 77--91, Lyon, France, September 2012. Am{\'e}lie
  Cordier, Marie Lefevre, Springer.
\newblock \doi{10.1007/978-3-642-32986-9\_8}.
\newblock URL \url{http://hal.inria.fr/hal-00735231}.

\bibitem[Dufour-Lussier et~al.(2013)Dufour-Lussier, Le~Ber, Lieber, and
  Martin]{avec_initiales_prenoms_dufourlussier:hal-00871703}
V.~Dufour-Lussier, F.~Le~Ber, J.~Lieber, and L.~Martin.
\newblock {Case Adaptation with Qualitative Algebras}.
\newblock In Francesca Rossi, editor, \emph{{International Joint Conferences on
  Artificial Intelligence (IJCAI-2013)}}, pages 3002--3006, P{\'e}kin, Chine,
  August 2013. AAAI Press.
\newblock URL \url{http://hal.inria.fr/hal-00871703}.

\bibitem[Dufour-Lussier et~al.(2014)Dufour-Lussier, Hermann, {Le~Ber}, and
  Lieber]{dufour2014_REVISOR_PCQA_short_version}
V.~Dufour-Lussier, A.~Hermann, F.~{Le~Ber}, and J.~Lieber.
\newblock {Belief revision in the propositional closure of a qualitative
  algebra}.
\newblock In \emph{14th International Conference on Principles of Knowledge
  Representation and Reasoning ({KR} 2014)}. {AAAI Press}, July 2014.

\bibitem[Gantner et~al.(2008)Gantner, Westphal, and W{\"o}lfl]{gantner08gqr}
Z.~Gantner, M.~Westphal, and S.~W{\"o}lfl.
\newblock {GQR} -- a fast reasoner for binary qualitative constraint calculi.
\newblock In \emph{{AAAI} Workshop on Spatial and Temporal Reasoning}, 2008.

\bibitem[Gerevini and Schubert(1995)]{gerivini95}
A.~Gerevini and L.~Schubert.
\newblock Efficient algorithms for qualitative reasoning about time.
\newblock \emph{Artificial Intelligence}, 74\penalty0 (1995):\penalty0
  207--248, 1995.

\bibitem[Hu\'e and Westphal(2012)]{hue12revising}
J.~Hu\'e and M.~Westphal.
\newblock Revising qualitative constraint networks: Definition and
  implementation.
\newblock In \emph{Tools with Artificial Intelligence ({ICTAI})}, pages
  548--555, 2012.
\newblock \doi{10.1109/ICTAI.2012.80}.

\bibitem[Katsuno and Mendelzon(1991{\natexlab{a}})]{katsuno-mendelzon:1991a}
H.~Katsuno and A.~Mendelzon.
\newblock {On the Difference Between Updating a Knowledge Base and Revising
  It}.
\newblock In James~F. Allen, Richard Fikes, and Erik Sandewall, editors,
  \emph{{{KR}'91: Principles of Knowledge Representation and Reasoning}}, pages
  387--394. Morgan Kaufmann, San Mateo, California, 1991{\natexlab{a}}.

\bibitem[Katsuno and Mendelzon(1991{\natexlab{b}})]{katsuno91}
H.~Katsuno and A.~Mendelzon.
\newblock {Propositional knowledge base revision and minimal change}.
\newblock \emph{{Artificial Intelligence}}, 52\penalty0 (3):\penalty0 263--294,
  1991{\natexlab{b}}.

\bibitem[Konieczny and P{\'e}rez(2002)]{konieczny02}
S.~Konieczny and R.~Pino P{\'e}rez.
\newblock {Merging information under constraints: a logical framework}.
\newblock \emph{{Journal of Logic and Computation}}, 12\penalty0 (5):\penalty0
  773--808, 2002.

\bibitem[Ligozat(1991)]{ligozat91generalized}
G.~Ligozat.
\newblock On generalized interval calculi.
\newblock In \emph{{AAAI}}, pages 234--240, 1991.

\bibitem[Ligozat and Renz(2004)]{ligozat04}
G.~Ligozat and J.~Renz.
\newblock {What Is a Qualitative Calculus? A General Framework}.
\newblock In C.~Zhang, H.W. Guesgen, and W.K. Yeaps, editors, \emph{PRICAI
  2004}, volume LNAI 3157, pages 53--64. Springer-Verlag, 2004.

\bibitem[Pearl(1984)]{mmt/Pear84a}
J.~Pearl.
\newblock \emph{{Heuristics -- {Intelligent} Search Strategies for Computer
  Problem Solving}}.
\newblock Addison-Wesley Publishing Co., Reading, MA, 1984.

\bibitem[Peppas(2008)]{peppa08}
P.~Peppas.
\newblock {Belief Revision}.
\newblock In F.~van Harmelen, V.~Lifschitz, and B.~Porter, editors,
  \emph{{Handbook of Knowledge Representation}}, chapter~8, pages 317--359.
  Elsevier, 2008.

\bibitem[Pujari et~al.(1999)Pujari, Kumari, and Sattar]{pujari99indu}
A.~K. Pujari, G.~V. Kumari, and A.~Sattar.
\newblock {{INDU}: An Interval \& Duration Network}.
\newblock In Norman Foo, editor, \emph{Advanced Topics in Artificial
  Intelligence}, volume 1747 of \emph{Lecture Notes in Computer Science}, pages
  291--303. Springer Berlin Heidelberg, 1999.
\newblock ISBN 978-3-540-66822-0.
\newblock \doi{10.1007/3-540-46695-9_25}.
\newblock URL \url{http://dx.doi.org/10.1007/3-540-46695-9_25}.

\bibitem[Randell et~al.(1992)Randell, Cui, and Cohn]{randell92spatial}
D.~Randell, Z.~Cui, and A.~G. Cohn.
\newblock {A spatial logic based on regions and connection}.
\newblock In \emph{Knowledge Representation}, pages 165--176, 1992.

\bibitem[Renz and Ligozat(2005)]{renz05}
J.~Renz and G.~Ligozat.
\newblock {Weak Composition for Qualitative Spatial and Temporal Reasoning}.
\newblock In P.~van Beek, editor, \emph{CP 2005}, LNCS 3709, pages 534--548.
  Springer-Verlag, 2005.

\bibitem[Stock(1997)]{stock97}
O.~Stock, editor.
\newblock \emph{{Spatial and Temporal Reasoning}}.
\newblock Kluwer Academic Publishers, 1997.

\bibitem[Vilain and Kautz(1986)]{ViKa86}
M.~B. Vilain and H.~Kautz.
\newblock Constraint propagation algorithms for temporal reasoning.
\newblock In \emph{Proceedings of the AAAI Conference on Artificial
  Intelligence (AAAI'86)}, pages 377--382, 1986.

\bibitem[Wallgr{\"u}n and Dylla(2010)]{wallgrun10relation}
Jan~Oliver Wallgr{\"u}n and Frank Dylla.
\newblock A relation-based merging operator for qualitative spatial data
  integration and conflict resolution.
\newblock Technical Report 022-06/2010, Transregional Collaborative Research
  Center SFB/TR 8 Spatial Cognition, 2010.

\end{thebibliography}

\end{document}